\documentclass[journal, onecolumn, draftclsnofoot]{IEEEtran}

\usepackage{cite}

\ifCLASSINFOpdf
\else
\fi

\usepackage{amsmath}

\usepackage{amssymb}
\usepackage{amsthm}
\usepackage{mathtools}
\usepackage{hyperref}

\usepackage{algorithmic}

\newcommand{\E}[1]{\mathbb{E}\left[#1\right]}
\newcommand{\Ep}[2]{\mathbb{E}_{#2}\left[#1\right]}

\newtheorem{theorem}{Theorem}[section]
\newtheorem{corollary}[theorem]{Corollary}
\newtheorem{lemma}[theorem]{Lemma}

\theoremstyle{definition}
\newtheorem{definition}{Definition}[section]

\usepackage{etoolbox}
\newcounter{equationstore}
\AtBeginEnvironment{proof}{\setcounter{equationstore}{\value{equation}}
\setcounter{equation}{0}}
\AtEndEnvironment{proof}{\setcounter{equation}{\value{equationstore}}}

\theoremstyle{remark}
\newtheorem*{remark}{Remark}

\begin{document}

\title{
	Optimal Tracking \\in Prediction with Expert Advice
}

\author{
	Hakan~Gokcesu 
	and~Suleyman~S.~Kozat~\IEEEmembership{}
	\thanks{
		This study is partially supported by Turkcell Technology within the framework of 5G and Beyond Joint Graduate Support Programme coordinated by Information and Communication Technologies Authority.
		
		This work is also supported in part by Outstanding Researcher Programme Turkish Academy of Sciences. 
		
		H. Gokcesu and S. S. Kozat are with the Department of Electrical and Electronics Engineering, Bilkent University, Ankara, Turkey; e-mail: \{hgokcesu, kozat\}@ee.bilkent.edu.tr
		
		H. Gokcesu is also with Turkcell Teknoloji, Maltepe, Istanbul, Türkiye.
		
		S. S. Kozat is also with DataBoss A.S., Bilkent Cyberpark, Ankara, Türkiye.
	}
}

\maketitle

\begin{abstract}
We study the prediction with expert advice setting, where the aim is to produce a decision by combining the decisions generated by a set of experts, e.g., independently running algorithms. We achieve the min-max optimal dynamic regret under the prediction with expert advice setting, i.e., we can compete against time-varying (not necessarily fixed) combinations of expert decisions in an optimal manner. 
Our end-algorithm is truly online with no prior information, such as the time horizon or loss ranges, which are commonly used by different algorithms in the literature. Both our regret guarantees and the min-max lower bounds are derived with the general consideration that the expert losses can have time-varying properties and are possibly unbounded. Our algorithm can be adapted for restrictive scenarios regarding both loss feedback and decision making. Our guarantees are universal, i.e., our end-algorithm can provide regret guarantee against any competitor sequence in a min-max optimal manner with logarithmic complexity. Note that, to our knowledge, for the prediction with expert advice problem, our algorithms are the first to produce such universally optimal, adaptive and truly online guarantees with no prior knowledge.
\end{abstract}

\newpage
\section{Introduction} \label{section:introduction}
This paper studies the problem of prediction with expert advice \cite{vovk1998game}, which is a widely studied topic in the fields of signal processing, machine learning, artificial intelligence, automatic control, and computational learning theory. Online algorithms produced for this problem setting can be readily used to combine (e.g. ensemble) predictors (i.e. `expert advice') in real-time, for various prediction tasks in numerous fields of science and engineering. Some examples of prediction tasks pertaining to signal processing and system automation include: adaptive filtering \cite{wung9099401,TSP_nonconvex_Projection}, predictive control \cite{nguyen9027879}, feature extraction \cite{haykin1315940,zhong9091107}, parameter estimation \cite{zigang1145726,nehorai372}, non-convex optimization \cite{gokcesu2021regret,TSP_nonconvex_MatrixCompletion}, classification \cite{gokcesu2021optimally,TSP_classification_Adaptive_Chien,TSP_classification_Kozat,TSP_classification_Adaptive_Hall}, time-series analysis \cite{TSP_time-series}, speech de-reverberation \cite{TSP_speech_dereverberation}.

In the prediction with expert advice setting, there exist a set of experts producing their own opinions regarding an underlying problem \cite{Gokcesu2020AGO}. The task assigned to the final predictor is to produce its own opinion after considering the opinions of these experts. We consider such a setting in an online scenario such that the experts produce their opinions at successive rounds, i.e. time indexes, for each round separately. Then, the role given to the final predictor is to ensemble these opinions to produce a joint opinion and update its `confidence' for each expert after evaluating the performances of these experts in an online manner. The evaluation of these performances is done via the loss function `revealed' by the environment, which could also be thought of as the adversary in a two-player game \cite{kobzar2020new}. Consequently, we consider this learning setup to be potentially adversarial, i.e., in the worst-case scenario, the loss functions can be generated so that any causal predictor cannot hope to perform better than a certain standard. This learning problem has been widely studied and many attempts to obtain ``optimal" algorithms have been made, even for only two or three experts, such as \cite{gravin2016towards} and \cite{drenska2020prediction}.

The performance evaluation of a predictor that learns which expert should be trusted is defined via the notion of regret, which is traditionally considered in a static manner. This regret measures the difference between two quantities, which are the overall loss (or error) the predictor experiences due to the opinions it produces and the overall loss experienced by the best performing expert, opinions of which produces the minimum overall loss among the set of experts. There is also a generalized definition of dynamic regret -which we focus on in this work- where the predictor competes against a sequence of weights for the opinions of all experts as opposed to the overall-best expert. Thus, the competitor weights attributed to the experts may change in time. The dynamic regret is a much more useful and challenging setting. Consider the simple fact that as the nature of an underlying problem evolves in time, the opinions of some experts may switch from low-loss to high-loss and vice-versa. As the specifications change, the level of ``expertise'' for each of the experts may also change in time. Note that what is described here, and in various works studying the notions of regret for both static and dynamic settings, is not specific to the problem of prediction with expert advice, but relates to the online learning problem in general. Furthermore, the notion of regret can be adapted so that the losses from latter time indices (i.e. the more recent ones) can be more pronounced, such as the concept of discounted losses \cite{chernov2010prediction}. Also note that alternative performance metrics are also investigated in the literature such as the consideration in \cite{chernov2010supermartingales}, where the aim is to show that the cumulative loss of the predictor is of the same order as the cumulative loss of the best expert. 

\newpage
State of the art works for dynamic regret generally study online convex optimization with general feasible decision sets \cite{mokhtari2016online,yang2016tracking,gokcesu_universal}, which are only required to be convex. The performance guarantees of such works do not translate well to the prediction with expert advice, where the decision set is a simplex, as the corresponding lower bounds in such settings are not intuitive, or even relevant, for the expert mixture framework. The reason for this lack of relevance is that, in general, at any time instance, the loss of each expert is realized individually. Hence, during the derivations of lower and upper bounds in the regret analyses, for each round, it is more meaningful to incorporate the range of expert losses, instead of the $\ell^2$ norm of the expert-loss vector, which aforementioned works consider inherently. There are also studies regarding dynamic regret specifically for the expert mixture framework \cite{herbster1998tracking,cesa2012new,Gokcesu2020RecursiveEA}. These works utilize prior knowledge such as the time horizon and loss ranges. They tend to be not adaptive to the loss dynamics or not have universal guarantees. Alternatively, the authors in \cite{koolen2014learning} state that they ``learn" the optimal ``learning rate", and the authors in \cite{cesa2007improved} propose algorithms adaptive to the loss ranges and translations. However, these works are only for the static setting and do not extend to the dynamic regret.

As we show in our work, simultaneous handling of such a variety of obstacles requires careful decision updates via biasing and/or scaling if necessary, possible subset reductions and multi-leveled approaches. 

Our results are applicable for all convex and, with some considerations, all non-convex loss functions. We do not require additional properties, e.g. mixable losses \cite{vovk2009prediction,chernov2009prediction}.

\subsection{Contributions}
Our contributions are as follows.
\begin{itemize}
	\item For the first time in the literature, we achieve a truly online and adaptive version of the exponential weighting method for the dynamic settings, with optimally adaptive translation-free and scale-free performance guarantees, without any hindsight information.
	\item We derive adversarial lower-bounds for the non-stationary expert scenarios with time-variant loss natures, as opposed to the conventionally investigated lower bounds for stationary environments.
	\item We demonstrate that the dynamic regret of our algorithm matches with the adaptive lower bounds, thus it is min-max optimal, i.e., it cannot be notably improved for adversarial scenarios. 
	\item We show that our algorithm can be adapted for feedback and decision restrictive scenarios.
	\item We efficiently extend our approach so that it has universality, or has universally the min-max optimal guarantees, i.e., it can match its guarantee against any competitor with the individually corresponding adversarial lower bound.
\end{itemize}

\newpage
\subsection{Outline}
This paper is constructed as follows, where each item corresponds to its respective section. The proofs of lemmas, theorems and corollaries are in the appendices.
\begin{enumerate}
	\item[\ref{section:introduction}:] We provided the topic, the state-of-the-art and our contributions.
	\item[\ref{section:problem_description}:] We formulate the problem of prediction with expert advice and describe its relations to some commonly tackled important problems in the literature.
	\item[\ref{section:adversarial_lower_bounds}:] We give our results for the adversarial min-max lower bounds, so that our algorithms can compare against.
	\item[\ref{section:exponential_weighting_with_uniform_mixing}:] We introduce our algorithm based on the exponential weights framework with uniform mixing, and we describe the adaptive selection of learning rates and mixture coefficients.
	\item[\ref{section:exponential_weighting_with_truncation}:] We describe an alternative approach to uniform mixing, namely truncation, and detail the advantages of both approaches.
	\item[\ref{section:restrictive_scenarios}:] We detail the restrictive scenarios in terms of both loss observations and distribution selections, and how our algorithm can be adapted for such scenarios.
	\item[\ref{section:universal_competition}:] We describe how our algorithm can be applied in a recursive, multi-level manner so that our performance guarantee against each competitor is individually optimal.
	\item[\ref{section:discounted_losses}] We utilize the scale-free property of our approach to demonstrate its application in a discounted loss scenario.
	\item[\ref{section:conclusion}:] We conclude with notable remarks.
	\item[\ref{appendix:adversarial_lower_bounds}-\ref{appendix:universal_competition}] Appendices \ref{appendix:adversarial_lower_bounds},\ref{appendix:exponential_weighting_with_uniform_mixing},\ref{appendix:exponential_weighting_with_truncation},\ref{appendix:universal_competition} provide proofs of various lemmas, theorems and corollaries for respective sections.
\end{enumerate}

\newpage
\section{Problem Description}
\label{section:problem_description}
We are working on the problem of prediction with expert advice. In this setting, our decisions are denoted as $p_t$, while the competitor decisions (possibly the very best) are denoted as $p_t^*$ at times $t=1,\ldots$. Both $p_t$ and $p_t^*$ are column vectors defined on a probability simplex. Effectively, they are distributions over $m=1,\ldots,M$ such that they are non-negative, i.e., $p_{t,m},p_{t,m}^*\geq 0$  for all $(t,m)$, and they individually sum to $1$, i.e., $$\sum_{m=1}^{M} p_{t,m} = \sum_{m=1}^{M} p_{t,m}^* = 1$$ for all $t$. Here, the secondary subscript (i.e. $m$) denotes the element index of the column vector. Furthermore, each such index $m$ corresponds to an expert. The entries $p_{t,m}$ and $p_{t,m}^*$ denote the weights (possibly selection probabilities or contribution coefficients) associated with each expert $m$ at time $t$ by us and the competitor sequence, respectively (for `opinions' ultimately generated for the underlying problem).

In this setting, we incur the regret
\begin{equation*}
	R_T = \sum_{t=1}^{T}\left( 
			f_t\left(\sum_{m=1}^{M} p_{t,m} x_{t}^{(m)}\right) - f_t\left(\sum_{m=1}^{M} p_{t,m}^* x_{t}^{(m)}\right)
		  \right),
\end{equation*}
where $x_{t}^{(m)}$ are the individual `opinions' produced by each expert $m$ at times $t$, 
and $f_t$ are a sequence of convex functions 
measuring the losses.
Note that the opinions $x_{t}^{(m)}$ are not necessarily scalar values; they belong to some opinion space and are possibly vectors. Consequently, at time $t$, the opinions of us and the competitor respectively become 
$x_t=\sum_{m=1}^{M} p_{t,m} x_{t}^{(m)}$ and $x_t^*=\sum_{m=1}^{M} p_{t,m}^* x_{t}^{(m)}$. 

Following that, from convexity of each $f_t$, we have
\begin{equation*}
	R_T \leq \sum_{t=1}^{T} g_t^\top \left(\sum_{m=1}^{M} p_{t,m} x_{t}^{(m)} - \sum_{m=1}^{M} p_{t,m}^* x_{t}^{(m)}\right),
\end{equation*}
where $g_t$ belongs to the partial derivative set $\partial f_t(x_t)$ with $x_t  = \sum_{m=1}^{M} p_{t,m} x_{t}^{(m)}$, i.e., our opinion for the underlying problem. After some rearrangements, we get
\begin{equation} \label{eq:vanilla_regret}
	R_T \leq \sum_{t=1}^{T} l_t^\top \left(p_{t} - p_{t}^*\right),
\end{equation}
where the individual surrogate loss incurred by each expert $m$ at time $t$ can be defined as $l_{t,m} \triangleq g_t^\top x_{t}^{(m)}$.

This problem definition is sufficiently general as observed from the following examples.

\newpage
\subsection{Base Example}
When $x_{t}^{(m)}$ are one-hot vectors such that only their $m^{th}$ element is non-zero and $1$, i.e. $x_{t,m}^{(m)} = 1$ and $x_{t,k}^{(m)} = 0$ for $k\neq m$, the regret becomes
\begin{equation*}
	R_T = \sum_{t=1}^{T}
		f_t\left(p_t\right) - f_t\left(p_t^*\right)
	,
\end{equation*}
which is an online convex optimization problem where the decision set is a $(M-1)$-simplex, i.e. $p_t, p_t^* \in \Delta_{M-1}$ and $\Delta_{M-1} \triangleq \{p | \sum_{m=1}^{M} (p)_m = 1; \forall m, (p)_m \geq 0\}$ where the number of experts $M$ is the length of $p_t$ and $p_t^*$. Here, $(p)_m$ denotes the $m^{th}$ element of the vector $p$ and the parenthesis is used to avoid confusion with $p_t$. 

\subsection{Non-convex Problems}
Suppose the losses incurred by each expert are arbitrary, and possibly we have non-convex losses $f_t$. We are interested in a scenario where the losses $f_t\left(x_{t}^{(m)}\right)$ incurred by the experts $m$ affect our overall loss in proportion with the probabilities $p_{t,m}$ that are assigned to them, i.e., we can still utilize \eqref{eq:vanilla_regret} with appropriately determined $l_t$. For example, if we randomly select $x_t$ as $x_t = x_{t}^{(m)}$ with probability $p_{t,m}$ for all $t,m$, then the regret (or its expected value for the specific example) equals to,
\begin{equation*}
	\mathcal{R}_T = \E{\sum_{t=1}^{T} \ell_t^\top (p_t - p_t^*)},
\end{equation*}
where $\ell_t$ is such that $\ell_{t,m} = f_t\left(x_{t}^{(m)}\right)$, $\E{\cdot}$ is the expectation operation due to the random selection of $x_t^{(m)}$ and their eventual (if any) effects on future losses and decisions.

\subsection{Some Definitions}
To observe how regret evolves, we will use a common approach, which is to investigate the successive changes between some `distances'. For that, we consider KL divergence, and have the following definitions.
\begin{definition}[Entropy] \label{definition:entropy}
	The entropy, specifically the information entropy, of a discrete probability distribution $p$, is denoted as $H(p)$ and computed as
	\begin{equation*}
		H(p) = -\sum_m p_m \log(p_m).
	\end{equation*} 
\end{definition}
\begin{definition}[KL Divergence] \label{definition:KL_divergence}
	The Kullback–Leibler divergence, or relative entropy, from $p$ to $q$, where $p$ and $q$ are discrete probability distributions, is denoted as $D(p\|q)$ and computed as
	\begin{equation*}
		D(p\|q) = \sum_{m} p_m \log\left(\frac{p_m}{q_m}\right).
	\end{equation*}
\end{definition}

In the next section, we investigate the adversarial (worst-case) regret lower bound for our learning problem and afterwards eventually show that our results are min-max optimal.

\newpage
\section{Adversarial Lower Bounds}
\label{section:adversarial_lower_bounds}
The adversarial lower bounds are generated considering that the sequence of loss vectors $\{l_t\}_{t=1}^T$ can be anything, i.e., even the worst-case scenario such that they may be generated by the environment in an adversarial manner. This means we cannot possibly guarantee to incur regret lower than a quantity, which is derived in this section.

\begin{lemma}[Two Expert Regret Lower Bound] \label{lemma:two_expert_regret_lower_bound}
	For the best fixed expert weights $p_{*,1}$ and $p_{*,2}$, there exist losses $l_{t,1}$ and $l_{t,2}$ such that 
	\begin{equation*}
		\sum_{t=1}^T \sum_{m=1}^2 l_{t,m}(p_{t,m} - p_{*,m}) \geq \frac{1}{\sqrt{2}} L_T^{(\infty,2)}, 
	\end{equation*}
	with $L_T^{(\infty,2)} = \sqrt{\sum_{t=1}^T \min_{\mu \in \mathbb{R}} \max_{1\leq k\leq 2} |l_{t,k}-\mu|^2}$, for any causally produced (e.g. by some algorithm) sequence of $\{p_t\}_{t=1}^T$. Here, $L_T^{(\infty,2)}$ is, in a sense, $\ell^2$-norm of the individual $\ell^\infty$-norms in a translation-free manner.
\end{lemma}

\begin{theorem}[Static Regret Lower Bound] \label{theorem:static_regret_lower_bound}
	There exist a sequence of loss vectors $\{l_t\}_{t=1}^T$ such that, for some fixed expert distribution $p_*$,
	\begin{equation*}
		\sum_{t=1}^T l_t^\top(p_t - p_*) \geq \frac{1}{\sqrt{2}}L_T^{(\infty,2)}\sqrt{\min\{\lfloor \log_2 M \rfloor,T\}}, 
	\end{equation*}
	with $L_T^{(\infty,2)} = \sqrt{\sum_{t=1}^T \min_{\mu \in \mathbb{R}} \max_{1\leq k\leq M} |l_{t,k}-\mu|^2}$, for any causally produced (e.g. algorithm) sequence of $\{p_t\}_{t=1}^T$. 
\end{theorem}

\begin{corollary}[Dynamic Regret Lower Bound] \label{corollary:dynamic_regret_lower_bound}
	There exist a sequence of loss vectors $\{l_t\}_{t=1}^T$, a competitor sequence $\{p_t^*\}_{t=1}^T$ with $0.5\sum_{t=1}^{T-1} \|p_{t+1}^* - p_t^*\|_1 \leq P_T$ such that,
	\begin{equation*}
		\sum_{t=1}^T l_t^\top(p_t - p_t^*) \geq \frac{1}{\sqrt{2}}L_T^{(\infty,2)}\sqrt{\min\{\lfloor P_T + 1\rfloor\lfloor \log_2 M \rfloor,T\}}, 
	\end{equation*}
	with $L_T^{(\infty,2)} = \sqrt{\sum_{t=1}^T \min_{\mu \in \mathbb{R}} \max_{1\leq k\leq M} |l_{t,k}-\mu|^2}$, for any causally produced (e.g. algorithm) sequence of $\{p_t\}_{t=1}^T$.
\end{corollary}

Next, we will show how to effectively bound our regret, starting with the traditional exponential weighting framework.

\newpage
\section{Exponential Weighting with Uniform Mixing} \label{section:exponential_weighting_with_uniform_mixing}
We first utilize an exponential weighting strategy, where the vanilla update equation is given as
\begin{equation*}
	p_{t+1,m} = \frac{p_{t,m}e^{-\eta_t l_{t,m}}}{\sum_{k=1}^{M}p_{t,k}e^{-\eta_t l_{t,k}}}.
\end{equation*}

Generally, these types of algorithms compete with fixed $p_t^* = p^*$. There are approaches to introduce switching competition where, for intervals $[t_k, t_{k+1})$, $p_t^* = p_k^*$. We are interested in arbitrary changes, where successive $p_t^*$ may never be equal. 

We modify the vanilla approach as
\begin{equation} \label{eq:upd_uni}
	p_{t+1,m} = (1-\gamma_t)\frac{p_{t,m}e^{-\eta_t l_{t,m}}}{\sum_{k=1}^{M}p_{t,k}e^{-\eta_t l_{t,k}}} + \gamma_t \frac{1}{M},
\end{equation}
where $\gamma_t$ is the coefficient for mixing with the uniform distribution, and thus, $0\leq \gamma_t\leq 1$. 

We utilize KL divergence by bounding the difference $D(p_t^* \| p_t) - D(p_t^* \| p_{t+1})$. For that, we have a presupposition regarding $\eta_t$, which will be satisfied later on.

\begin{lemma}[KL Lower Bound]\label{lemma:KL}
	If we employ \eqref{eq:upd_uni} in our distribution updates with $\eta_t > 0$, we have
	\begin{align*}
		\frac{1}{\eta_t} &\left[D(p_t^* \| p_t) - D(p_t^* \| p_{t+1})\right]
		 \geq l_t^\top(p_t - p_t^*) -\eta_t \Ep{(l_{t,m} - \mu_t)^2}{p_t} + \frac{\log(1-\gamma_t)}{\eta_t},
	\end{align*}
	when $|\eta_t(l_{t,m} - \mu_t)| \leq 1$ for all $t$ and $m$,
	where $\mu_t = \Ep{l_{t,m}}{p_t}$ and $\Ep{\cdot}{p_t}$ is such that $\Ep{X_m}{p_t} = \sum_{m=1}^{M} p_{t,m} X_m$, similar to an expectation operator with distribution $p_t$.
\end{lemma}

\begin{theorem}[Initial Regret Bound]\label{theorem:initial_regret_bound}
	The regret in \eqref{eq:vanilla_regret} can be bounded as
	\begin{align*}
		R_T &\leq \sum_{t=1}^{T} l_t^\top (p_t - p_t^*) 
		\begin{multlined}
			\leq \sum_{t=1}^{T} \eta_t \Ep{(l_{t,m} - \mu_t)^2}{p_t}
			+\frac{D(p_t^* \| p_t) - D(p_t^* \| p_{t+1})}{\eta_t}
			-\frac{\log(1-\gamma_t)}{\eta_t}.
		\end{multlined}
	\end{align*}
	where 
	\begin{itemize}
		\item $\eta_t \geq \eta_\tau > 0$ whenever $\tau > t$;
		\item $\eta_t|l_{t,m} - \mu_t| \leq 1$ for all $t,m$;
		\item 
		$\Ep{(l_{t,m}-\mu_t)^2}{p_t} = \sum_{m=1}^{M} p_{t,m} \left(l_{t,m} - \mu_t\right)^2$;
		\item $\mu_t = \Ep{l_{t,m}}{p_t} = \sum_{m=1}^{M} p_{t,m} l_{t,m}$.
	\end{itemize}
\end{theorem}
\begin{proof}
	The claim is achieved by summing the inequality in Lemma \ref{lemma:KL} over $t$ and some rearrangements.
\end{proof}

\newpage
To further simplify the bound in Theorem \ref{theorem:initial_regret_bound}, we need to replace $D(p_t^* \| p_{t+1})$ with $D(p_{t+1}^* \| p_{t+1})$ so that we can apply a telescoping summation.

\begin{lemma}[KL Bound] \label{lemma:change_KL}
	\begin{align*}
		-D(p_t^* \| p_{t+1}) \leq -&D(p_{t+1}^* \| p_{t+1}) + H(p_t^*) - H(p_{t+1}^*) 
		+\frac{1}{2}\|p_{t+1}^* - p_{t}^*\|_1 \cdot \log\left(\frac{1}{\underset{1\leq m\leq M}{\min} p_{t+1,m}}\right),
	\end{align*}
	where $\|\cdot\|_1$ is the $\ell^1$-norm operator with $\|X\|_1 = \sum_m |X_m|$.
\end{lemma}

\begin{corollary}[Parametric Regret Bound] \label{corollary:parametric_regret_bound}
	Combining Theorem \ref{theorem:initial_regret_bound} with Lemma \ref{lemma:change_KL},
	\begin{align*}
		R_T \leq \sum_{t=1}^{T} \Bigg(&\eta_t \Ep{(l_{t,m} - \mu_t)^2}{p_t} +\frac{H(p_t^*) - H(p_{t+1}^*)}{\eta_t}
		\\& +\frac{D(p_t^* \| p_t) - D(p_{t+1}^* \| p_{t+1})}{\eta_t} -\frac{\log(1-\gamma_t)}{\eta_t}
		 +\frac{\|p_{t+1}^* - p_t^*\|_1 \log\left(M/\gamma_t\right)}{2\eta_t}\Bigg),
	\end{align*}
	since, from \eqref{eq:upd_uni}, $p_{t+1,m} \geq \gamma_t/M$ in a tight manner.
\end{corollary}

Last component in the right-hand side sum of Corollary \ref{corollary:parametric_regret_bound} has multipliers $\|p_{t+1}^* - p_t^*\|_1$ which can be arbitrary. Thus, there is no effective way to reduce it without upper bounding $1/\eta_t$ and $\log(M/\gamma_t)$. The natural upper-bound for $1/\eta_t$ is $1/\eta_T$ since $\eta_t$ is positive non-increasing.

\begin{corollary}[Non-increasing $\eta_t$, $\gamma_t$] \label{corollary:non_increasing}
	In Corollary \ref{corollary:parametric_regret_bound}, selecting $\eta_t$ and $\gamma_t$ such that they do not increase, we obtain
	\begin{align*}
		&\begin{aligned}
			\mathllap{R_T \leq} \sum_{t=1}^{T} &\left(\eta_t \Ep{(l_{t,m} - \mu_t)^2}{p_t} -\frac{\log(1-\gamma_t)}{\eta_t}\right)
		\end{aligned}
		\\&+\frac{\log\left(M/\gamma_{T-1}\right) \sum_{t=1}^{T-1}\|p_{t+1}^* - p_t^*\|_1}{2\eta_T}
		+\frac{\log M + \log(M/\gamma_T)}{\eta_T}
		.
	\end{align*}
\end{corollary}

The mixer $\gamma_t$ needs to scale with a polynomial of $T$, otherwise another component $\log(1-\gamma_t)$ has inefficient bounds. Most natural and effective selection appears to be $\gamma_t = (t+1)^{-1}$.

\begin{corollary}[Setting $\gamma_t$]\label{corollary:setting_gamma_t}
	In Corollary \ref{corollary:non_increasing}, when we set $\gamma_t = (t+1)^{-1}$, we get
	\begin{align*}
		R_T \leq &\left(\sum_{t=1}^{T} \eta_t \Ep{(l_{t,m} - \mu_t)^2}{p_t}\right) +\frac{2\log(M(T+1))}{\eta_T}
		+\frac{\log\left(MT\right) \sum_{t=1}^{T-1}\|p_{t+1}^* - p_t^*\|_1}{2\eta_T}
		.
	\end{align*}
\end{corollary}

\newpage
Now, we can optimize learning rate $\eta_t$.

\begin{theorem}[Learning Rate Optimization]\label{theorem:learning_rate_optimization}
	Using Corollary \ref{corollary:setting_gamma_t}, if we set
	\begin{equation*}
		\eta_t = \min\left\{\widetilde{\eta}_t, 1/E_t\right\},
	\end{equation*}
	\begin{align*}
		\widetilde{\eta}_t = \sqrt{\frac{\log(M(T+1))(2 + P_T)}{2V_t}}
		,\; E_t =  \underset{{1\leq m\leq M}}{\max_{1\leq \tau\leq t,}} |l_{\tau,m} - \mu_\tau|,
	\end{align*}
	the regret becomes
	\begin{equation*}
		R_T \leq 2\sqrt{2(2 + P_T)V_T\log(M(T+1))} 
			+\log(M(T+1)) (2+P_T) E_T,
	\end{equation*}
	where $P_T = 0.5\sum_{t=1}^{T-1} \|p_{t+1}^* - p_t^*\|_1$ is the overall sum of all the total variation distances between successive competitions $p_t^*$, $\mu_t = \sum_{m=1}^{M} p_{t,m} l_{t,m}$ and $V_t = \sum_{\tau=1}^{t} \Ep{(l_{\tau,m} - \mu_\tau)^2}{p_\tau}$ is the sum of loss variances for $p_\tau$ with $\tau=1,\ldots,t$. 
\end{theorem}

We also have an alternative result, which eliminates the term relating to the loss range $E_T$.
\begin{corollary} \label{corollary:minimum_biasing}
	Using $\eta_t = \sqrt{\log(M(T+1))(2 + P_T)/Q_t}$ instead, we can obtain an alternative bound
	\begin{equation*}
		R_T \leq 2\sqrt{(2 + P_T)Q_T\log(M(T+1))},
	\end{equation*}
	where, for $z_t = \min_{1\leq m\leq M} l_{t,m}$, $$Q_t = \sum_{\tau=1}^t \Ep{(l_{\tau,m} - z_\tau)^2}{p_\tau}.$$
\end{corollary}

Note that, no prior information here for truly online behavior, namely $T$ and $P_T$, would lead to non-optimal setting of $\eta_t$. Furthermore, the regret has poly-logarithmic multiplicative redundancy compared to the lower bound in Corollary \ref{corollary:dynamic_regret_lower_bound}. We solve these issue starting with the following section.

\newpage
\section{Exponential Weighting with Truncation}
\label{section:exponential_weighting_with_truncation}
Here, we introduce an alternative update method by modifying our approach in Section \ref{section:exponential_weighting_with_uniform_mixing}. This opens up the capability to achieve improved regret guarantees by eliminating the multiplicative redundancy gap dependent on the time horizon $T$. Specifically, we update from $p_t$ to $p_{t+1}$ as follows.
\begin{align}\label{eq:upd_proj}
	q_{t+1,m} &= \frac{p_{t,m}e^{-\eta_t l_{t,m}}}{\sum_{k=1}^{M}p_{t,k}e^{-\eta_t l_{t,k}}}, \nonumber
	\\p_{t+1} &= \Pi_{a,b}\left(q_{t+1}\right),
\end{align}
where $\Pi_{a,b}$ is the truncation operation such that for all $m$,
\begin{equation}\label{eq:truncate}
	p_{t+1,m} = \begin{cases}
		a&, \text{ if } (\sigma q_{t+1,m}) < a;
		\\b&, \text{ if } (\sigma q_{t+1,m}) > b;
		\\\sigma q_{t+1,m}&, \text{ otherwise}.
	\end{cases}
\end{equation}
i.e., given a suitable pair of $(a,b)$, we find the appropriate $\sigma$ which results in a probability distribution for $p_{t+1}$ in \eqref{eq:truncate}. Here, $(a,b)$ necessarily satisfy $0\leq a\leq 1/M\leq b\leq 1$.

Note that we have eliminated the uniform mixing term $\gamma_t/M$ and effectively set $\gamma_t=0$. Uniform mixture was an efficient way to handle the required intermediate step of lower bounding our selected probabilities in the regret analysis. Replacing uniform mixture with truncation may not be as computationally efficient, but has flexibility in addition to certain other benefits. Specifically, we can completely eliminate the $\log(T)$ term from the regret guarantees. We show how this truncation operation integrates to our regret analysis.

\begin{lemma}\label{lemma:truncation}
	For $a\leq p_t^*\leq b$, where the inequalities are element-wise for the pair of scalars $(a,b)$, 
	\begin{equation*}
		D(p_t^* \| q_{t+1}) \geq D(p_t^* \| p_{t+1}),
	\end{equation*}
	where the distributions $p_{t+1}$ and $q_{t+1}$ are related as in \eqref{eq:upd_proj}.
\end{lemma}

The regret guarantee corresponding to truncation (as opposed to mixing) is as follows.

\begin{theorem}[Regret for Truncation] \label{theorem:vanilla_truncation}
	Considering the scenario that $p_t,p_t^* \geq a$ and $p_t,p_t^* \leq b$ for all $t$; when the exponential weighting learning rate is
	\begin{align*}
		\eta_t = \min\left\{\widetilde{\eta}_t, 1/E_t\right\}, \quad&\widetilde{\eta}_t = \sqrt{\frac{\log(M/a)+\log(1/a)P_T}{2V_t}},
		\quad E_t =  \underset{{1\leq m\leq M}}{\max_{1\leq \tau\leq t,}} |l_{\tau,m} - \mu_\tau|,
	\end{align*}
	the dynamic regret becomes 
	\begin{equation*}
		R_T \leq 2\sqrt{2(\log(M/a)+\log(1/a)P_T)V_T} 
		+(\log(M/a)+\log(1/a)P_T)E_T,
	\end{equation*}
	with the same general assumptions as in Section \ref{section:exponential_weighting_with_uniform_mixing}. Specifically, $P_T = 0.5\sum_{t=1}^{T-1} \|p_{t+1}^* - p_t^*\|_1$ is the overall sum of all the total variation distances between successive competitions $p_t^*$, and $V_t = \sum_{\tau=1}^{t} \Ep{(l_{\tau,m} - \mu_\tau)^2}{p_\tau}$ is the sum of loss variances with respects to $p_\tau$ from $\tau=1$ to $\tau=t$, and $E_t =  \max_{1\leq \tau\leq t,1\leq m\leq M} |l_{\tau,m} - \mu_\tau|$ can be bounded by the general range diameter of losses, i.e. $E=\max_{1\leq t\leq T,1\leq m\leq n\leq M} |l_{t,m} - l_{t,n}|$.
\end{theorem}

An alternative bound similar with Corollary \ref{corollary:minimum_biasing} is also available here.
\begin{corollary} \label{corollary:alternative_truncation}
	If $\eta_t = \sqrt{(\log(M/a)+\log(1/a)P_T)/Q_t}$ instead, we can obtain the alternative bound
	\begin{equation*}
		R_T \leq 2\sqrt{(\log(M/a)+\log(1/a)P_T)Q_T},
	\end{equation*}
	where, for $z_t = \min_{1\leq m\leq M} l_{t,m}$, $$Q_t = \sum_{\tau=1}^t \Ep{(l_{\tau,m} - z_\tau)^2}{p_\tau},$$
	and $0< a\leq p_t^*$.
\end{corollary}
\begin{proof}
	The proof follows from Lemma \ref{lemma:truncation} with arguments similar to the ones for Corollary \ref{corollary:minimum_biasing} and Theorem \ref{theorem:vanilla_truncation}.
\end{proof}

Thus, using Lemma \ref{lemma:truncation}, we obtain similar but better guarantees compared to Theorem \ref{theorem:learning_rate_optimization} and Corollary \ref{corollary:minimum_biasing} after setting $a=\alpha/M$ and $b=(1-\alpha)+\alpha/M$ for some $0\leq \alpha\leq 1$, and eliminate the $\log(T)$ terms. Originally, we cannot assume that $p_{t,m},p_{t,m}^* \in [a,b]$ for any $t,m$. To circumvent this, we apply a one-to-one mapping to the original decisions $p_t$ using $d=(1-\alpha)p+\alpha u$ where $\alpha = a M$ and $u$ is the uniform distribution. Afterwards, since $d_{t,m},d_{t,m}^* \in [a,b]$, we can learn $d_t$ as in Corollary \ref{corollary:alternative_truncation}, and extract the true decision $p_t = (1/(1-\alpha))(d_t - \alpha u)$. Next, we give the relevant results.

\newpage
\begin{corollary} \label{corollary:eliminate_logT}
	The regret for probability distribution $p_t$ against $p_t^*$, with each entry belonging to $[0,1]$ and the cumulative total variation distance $P_T = 0.5\sum_{t=1}^{T-1} \|p_{t+1}^* - p_t^*\|_1$, can also be obtained as, for some $c\leq 5.111$, $\theta\geq 0.3$,
	\begin{equation*}
		R_T \leq c\sqrt{\left(1+(1-\theta)P_T\right)Q_T^d\log(M)},
	\end{equation*}
	after a one-to-one switch from $p_t,p_t^*$ to $d_t = (1-\alpha) p_t + \alpha u$, $d_t^*= (1-\alpha)p_t^* + \alpha u$ with a mixer $0\leq\alpha\leq1$ such that 
	$$\alpha = \widehat{\alpha} = \arg\min_{0<\alpha<1} \left(\frac{\log(M^2/\alpha)}{(1-\alpha)^2}+\frac{\log(M/\alpha)}{1-\alpha}P_T\right),$$
	or, alternatively,
	$$\alpha = \alpha^* = \frac{-0.5}{W_{-1}\left(-0.5e^{-0.5}M^{-2}\right)},$$
	where $W_{-1}$ is the Lambert-W function, $u$ is the uniform distribution with $(u)_{m} = 1/M$ for all $m$, and $d_t$ is learned using the truncated exponential weighting in Corollary \ref{corollary:alternative_truncation}.
\end{corollary}

In Corollary \ref{corollary:eliminate_logT}, the multiplicative redundancy by setting $\alpha=0.5$ is $20\%$ worse, since for $\lfloor 1+P_T \rceil$ in Corollary \ref{corollary:dynamic_regret_lower_bound}, $\lfloor 1+P_T \rfloor \geq \max\{1,P_T\} \geq (\zeta+P_T)/(1+\zeta)$ with $\zeta > 0$ is tight for arbitrary $P_T$.

\begin{remark}[Min-Max Optimal]
	Theorem \ref{theorem:vanilla_truncation} and Corollary \ref{corollary:dynamic_regret_lower_bound} show that our algorithm is min-max optimal when $P_T$ is $O(L_T^{(\infty,2)} /(E_T^2 \log M))$ if the average losses of the rounds $\mu_t$ is used. If the minimum losses of the rounds $z_t$ is used instead, as in Corollary \ref{corollary:alternative_truncation}, the algorithm is min-max optimal for all $P_T$. However, this may not mean that it is more advantageous to use $z_t$ instead of $\mu_t$. Note that the min-max lower-bound analysis only considers the worst-case scenarios, where the expert losses belong to one of the two extremes. If they are more uniformly distributed across the loss range, the regret with a variance component can potentially be lower. 
\end{remark}

\newpage
\section{Restrictive Scenarios}
\label{section:restrictive_scenarios}
In this section, we investigate some restrictive scenarios with certain limitations on either the feedback or decisions. Note that the second additive term in the regret bounds here can be eliminated if needed by biasing with the minimum instead of the mean as explained in the previous sections.

\subsection{Feedback Limitations}
We mainly investigate the feedback limitation scenarios when our observation of $l_t$ are limited, i.e. we observe $\widetilde{l}_t$ instead of $l_t$. 

\begin{corollary}[Noisy Observations] \label{corollary:noisy_observations}
	Suppose only a noisy version of the losses $l_t$, i.e. $\widetilde{l}_t$, is revealed with $\Ep{\widetilde{l}_t}{t} = l_t$ where $\Ep{\cdot}{t}$ is the conditional expectation of the noise. Then,
	\begin{align*}
		\E{R_T} 
		&= \E{\sum_{t=1}^T l_t^\top (p_t - p_t^*)} 
		\leq \E{\sum_{t=1}^T (\widetilde{l}_t)^\top(p_t - p_t^*)} 
		\\&\begin{multlined}
			\leq \E{C_1\sqrt{(1+P_T) \log(MT)\widetilde{V}_T}} 
			+ \E{C_2\log(MT)P_T\widetilde{E}_T},
		\end{multlined}
		\\&\begin{multlined}
			\leq C_1\sqrt{(1+P_T) \log(MT)\E{\widetilde{V}_T}} 
			+ C_2\log(MT)P_T\E{\widetilde{E}_T},
		\end{multlined}
	\end{align*}
	for some positive constants $C_1, C_2$, where
	\begin{align*}
		\widetilde{V}_T 
		&= \sum_{t=1}^{T}\sum_{k=1}^{M} p_{t,k} \left(\widetilde{l}_{t,k} - \sum_{k=1}^{M} p_{t,k} \widetilde{l}_{t,k}\right)^2,
		\qquad
		\widetilde{E}_T
		= \max_{1\leq t\leq T, 1\leq m\leq M} \left|\widetilde{l}_{t,m} - \sum_{k=1}^{M} p_{t,k} \widetilde{l}_{t,k}\right|,
	\end{align*}
	and the algorithm corresponding to Theorem \ref{theorem:learning_rate_optimization} is utilized. Similar results for other algorithms also exist.
\end{corollary}
\begin{proof}
	This is a direct corollary of Theorem \ref{theorem:learning_rate_optimization}, where the algorithm uses $\widetilde{l}_t$ instead of $l_t$.
\end{proof}

In Corollary \ref{corollary:noisy_observations}, the quantity $\widetilde{E}_T$ can be possibly too large, on the order of $\widetilde{V}_T$ after upper-bounding the $\max$ operation with a sum. Then, employing the minimum biasing mentioned in the previous section may actually improve the bound up-to an order of $\sqrt{\log(MT)}$. Similar arguments for the methods in Section \ref{section:exponential_weighting_with_truncation} also translates; however, some issues may arise in bandit-feedback scenarios for Corollary\ref{corollary:eliminate_logT}.

\newpage
\subsection{Decision Limitations}
We investigate the cases that $p_t$ (and also $p_t^*$) cannot be any member of the $M$-dimensional simplex $\Delta_M$ but are restricted to a subset $\Delta_M^{(p)}$ such that $\Delta_M^{(p)} = \{p': p' \geq p\}$, where the inequality is element-wise and $\|p\|_1<1$. This generalizes the original setting, which is equivalent to the scenario that $p$ is all-zero. We can treat this limitation also as a predetermined allocation of $\|p\|_1$ part of the weights, where $\|\cdot\|_1$ is the $\ell^1$-norm.

\begin{corollary}[Lower Bounded Decisions]
\label{corollary:lower_bound_decisions}
	If $p_t \geq p$ for some $p$, where the inequality is element-wise, then we instead learn some $q_t$ such that $p_t = p + (1-\|p\|_1)q_t$. Then,
	\begin{align*}
		R_T 
		&= \sum_{t=1}^{T} l_t^\top(p_t - p_t^*)
		= \left(1- \|p\|_1\right)\sum_{t=1}^T \left(l_t \right)^\top (q_t - q_t^*),
	\end{align*}
	i.e., all corresponding results from previous theorems and corollaries are scaled with $(1-\|p\|_1)$. 
	
	Since total variation distances are scaled by $1/(1-\|p\|_1)$ from $p_t^*$ to $q_t^*$, the optimal square-root regret terms are scaled with $\sqrt{1-\|p\|_1}$ while the secondary terms are unaffected.
\end{corollary}

\subsection{Bandit Problems}  
The bandit problems can also be handled by our algorithm. Suppose the selection probabilities of the bandit arms are $b_{t,m}$. Then, we use the following losses for each $m$.
\begin{equation*}
	\widetilde{l}_{t,m} = \begin{cases}
		l_{t,m}/b_{t,m}	&, \text{ if $m$ is observed;}  
		\\0				&, \text{ otherwise.} 
	\end{cases}
\end{equation*}
For the standard multi-arm bandit problem $b_{t,m} = p_{t,m}$.
Now, suppose the semi-bandit scenario such that we select and observe $K$ out of $M$ arms (with possibly $K=1$). Thus, $\|b_t\|_1 = K$. Then, $b_{t,m} = K p_{t,m}$ and $p_{t,m} \leq K^{-1}$ for all $m$. We can learn $p_t$ using the truncation from Section \ref{section:exponential_weighting_with_truncation} with the upper-bound $K^{-1}$. However, Corollary \ref{corollary:eliminate_logT} may bring issues with its mismatched probabilities. Thus, we set $a=0$ and instead use the adaptive mixture in Section \ref{section:exponential_weighting_with_uniform_mixing}. Then, in accordance with Sections \ref{section:exponential_weighting_with_uniform_mixing} and \ref{section:exponential_weighting_with_truncation}, under the consideration that entries of $b_t^*$ are binary ($0$ or $1$), our regret guarantee scales with the total number of arm switches as opposed to the state-of-the-art which scales with the total number of $K$-arm group switches \cite{vural2019minimax} for non-negative arm losses. In other words, we can count changes in the competitor's $K$-arm selection partially, if only one arm is replaced then it counts as a fractional switch of $1/K$ in accordance with the definition of switches widely discussed in the literature \cite{Haoyu_Wei_2020}. 

\newpage
\section{Universal Competition}
\label{section:universal_competition}
\subsection{Path Adaptive Learning Rates}
To achieve universal regret guarantees, we first adapt the algorithm so that the learning rates utilizes not a fixed $P_T$ but a time-growing $P_T$. To achieve this, we employ a learning rate resetting. 
\begin{corollary}[Doubling Rates] \label{corollary:doubling_rates}
	Given, for some positive integer $K$, $2^{K-1}\leq 1+P_T< 2^K$ and
	applying the method from Corollary \ref{corollary:eliminate_logT}, for $k\in\{1,\ldots,K+1\}$, we reset the learning right before times $t=2^{k-1}$ and use $P_k = 2^{k-1}-1$ for the following duration of $2^{k-1}$ rounds until $t=2^{K}$, after which we reset with $P_{K+1}=2^{K}-1$ and remain indefinitely. Then, the regret for $P_T = 0.5\sum_{t=1}^{T-1} \|p_{t+1}^*-p_t^*\|_1$ becomes, for some $c'\leq 8.4$ and $\theta'\leq 1.2$,
	\begin{equation*}
		R_T \leq c'\sqrt{\left(1+(1+\theta')P_T\right)Q_T^d\log(M)},
	\end{equation*}
	where $Q_t^d = \sum_{\tau=1}^t \Ep{(l_{\tau,m} - z_\tau)^2}{d_\tau}$ for $z_t = \underset{1\leq m\leq M}{\min} l_{t,m}$, $M$ is the number of experts, and $\log$ is natural logarithm.
\end{corollary}

To achieve universal competition property, our algorithm should be able to work without the prior knowledge of $P_T$. To achieve that we utilize a doubling trick by running multiple copies of the version in Corollary \ref{corollary:doubling_rates}, with the same losses, with each using $P_T = 2^m-1$ for growing numbers $m$ from $m=0$ until $m=\log(2t)$. Note that, as time goes, higher $P_T$ can be generated (branched out) from the already running copies. These hyper-experts are combined with the same algorithm without resetting, i.e. the one from Corollary \ref{corollary:eliminate_logT}, at a higher level as individual experts with this hyper-algorithm run setting $P_T=0$. To handle the growing number of such hyper-experts as time progresses, we recursively mix them two at a time using our algorithm as a hyper-algorithm. Thus, the hyper-expert which sets $P_T=0$ is mixed with a hyper-algorithm which mixes the hyper-expert which sets $P_T=1$ and another hyper-algorithm, and so on. This construction allows the gradual addition of new experts without invalidating the nature of truly online learning, i.e. prior knowledge of $T$. Note that if not done in this manner but the unknown time horizon $T$ is solely handled with doubling tricks, the overall regret incurred would have redundant non-constant time-growing components.

\begin{figure}
	\begin{algorithmic}[1]
		\STATE \textbf{Inputs: } the number of original experts $M$.
		\STATE \textbf{Memory: } the set of runs $\mathcal{A}$ and the set of mixers $\mathcal{B}$.
		$$\begin{cases}
			\text{Each } A_k\in\mathcal{A} \text{ ensembles the } M \text{ original experts},
			\\\text{Each } B_{k-1}\in\mathcal{B} \text{ ensembles hyper-experts } \{A_{k},B_{k}\}.
		\end{cases}$$
		\STATE \textbf{Initialize: } $\mathcal{A}=\emptyset$ and $\mathcal{B}=\emptyset$.
		\\\textbf{RUNTIME: }
		\STATE Create the run $A_1\in \mathcal{A}$ for combining $M$ original experts.
		$A_1$: resets learning rates (Corollary \ref{corollary:doubling_rates}), $P_T=0$. 
		\\ (Initialize $p_1 \in A_1$, e.g.  $p_{1,m} = (1/M)$ for $1\leq m\leq M$)
		\STATE Create mixers $B_0,B_1\in \mathcal{B}$ for combining $2$ hyper-experts. 
		 $B_0$: no reset (Corollary \ref{corollary:eliminate_logT}), $P_T=0$.
		\\ (Initialize $p_1 \in B_0$, e.g. $p_{1,1}=p_{1,2}=0.5$)
		\qquad\qquad\qquad $B_1$: directly copies outputs of $A_1$ (for now).
		\STATE Start at $t=1$, $k=1$.
		\WHILE{NOT terminated}
			\IF{$t=2^{k}$}
				\STATE Create $A_{k+1}\in\mathcal{A}$ for $P_T=2^k-1$ (Corollary \ref{corollary:doubling_rates}).
				 (Treat past progression of $A_{k+1}$ as a copy of $A_k$).
				\STATE Create $B_{k+1}\in\mathcal{B}$, it copies $A_{k+1}$ for now. 
				\STATE Overwrite $B_k$ as the mixer of $\{A_{k+1},B_{k+1}\}$.
				 $B_k$ starts at time $t$ with $P_T=0$ (Corollary \ref{corollary:eliminate_logT}).
				\STATE Target the next segment: $k \leftarrow (k+1)$.
			\ENDIF
			\STATE Observe $l_t$, i.e. $l_{t,m}$ for $1\leq m\leq M$, incur loss $l_t^\top p_t$.
			\STATE Denote the losses of $A_k$,$B_k$ as $l_{t,(A_{k})}$,$l_{t,(B_{k})}$.
			\STATE Denote the decisions of $A_k$,$B_k$ as $p_{t}^{(A_{k})}$,$p_{t}^{(B_{k})}$.
			\STATE Denote the learning rates of $A_k$,$B_k$ as $\eta_{t}^{(A_{k})}$,$\eta_{t}^{(B_{k})}$.
			\STATE Compute the losses $l_t^\top p_t^{(A_k)}$ for all $A_k\in\mathcal{A}$.
			\STATE Loss incurred by $B_k$ is equal to $A_k$.
			\STATE Set $i=1$.
			\WHILE{$i<=k$}
				\STATE Define the loss vector seen by $B_{k-i}$ as $l_{t}^{(B_{k-i})} = [l_{t,(A_{k+1-i})},l_{t,(B_{k+1-i})}]^\top.$
				\STATE Compute the loss incurred by $B_{k-i}$ as $l_{t,(B_{k-i})} = p_{t,1}^{(B_{k-i})} l_{t,(A_{k+1-i})} + p_{t,2}^{(B_{k-i})} l_{t,(B_{k+1-i})}.$
				\STATE Update $p_t^{(A_{k+1-i})},\eta_t^{(A_{k+1-i})}$ for $A_{k+1-i}$.
				\STATE Update $p_t^{(B_{k-i})},\eta_t^{(B_{k-i})}$ for $B_{k-i}$.
			\ENDWHILE
			\STATE Get $p_{t+1} = \sum_{i=1}^k \left(p_{t+1,1}^{(B_{i-1})}\prod_{j=0}^{i-2} p_{t+1,2}^{(B_{j})}\right) p_{t+1}^{(A_i)}$.
			\STATE Go to next round: $t \leftarrow (t+1)$.
		\ENDWHILE
	\end{algorithmic}
	\caption{Algorithm UTEW (Universal Truncated Exponential Weighting)}
	\label{algorithm:universal}
\end{figure}

\newpage
\begin{theorem}[UTEW Regret]\label{theorem:universal_regret}
 The regret incurred by the universal algorithm UTEW (universal truncated exponential weighting) as given in \figurename\ref{algorithm:universal} is, for $c'\leq 75.17$,
 \begin{equation*}
 	R_T \leq C\sqrt{\lfloor1+P_T\rfloor\log(M)}L_T^{(\infty,2)},
 \end{equation*}
	where $P_T = 0.5\sum_{t=1}^{T-1} \|p_{t+1}^* - p_t^*\|_1$ is the unknown cumulative total variation distance, $M$ is the number of experts and $$L_T^{(\infty,2)} = \sqrt{\sum_{t=1}^T \min_{\mu \in \mathbb{R}} \max_{1\leq k\leq M} |l_{t,k}-\mu|^2}$$ measures the adaptive loss deviations.
\end{theorem}

\begin{remark}
	Algorithm UTEW in \figurename\ref{algorithm:universal} is the first in literature to achieve guarantees which are truly online, adaptive and min-max optimal in the exact sense, in accordance with Theorem \ref{theorem:universal_regret} and Corollary \ref{corollary:dynamic_regret_lower_bound}, where $R_T$ always satisfies $O(T)$.
\end{remark} 

\begin{remark}
	Complexities (time \& memory) of Algorithm UTEW in \figurename\ref{algorithm:universal} is $O(\log T)$. Improvements to time, computation and memory complexities are possible, similar to \cite{gokcesu_universal}.
\end{remark}

\begin{corollary}[Domain Universality]
	For all $t$ and $m$, if $p_{t,m}^* \geq \phi\cdot(p^*)_m$ for $0\leq\phi\leq1$, where the distribution $p^*$ is unknown, the regret becomes (for some $C',C''\leq C$),
	\begin{align*}
		R_T \leq &\sqrt{1-\phi}C''\sqrt{\left\lfloor1+P_T\right\rfloor\log(M)}L_T^{(\infty,2)}
		+\phi C'\sqrt{\log(M)}L_T^{(\infty,2)}.
	\end{align*}
\end{corollary}
\begin{proof}
	For some $q_t^*$, $p_t^* = \phi p^* + (1-\phi) q_t^*$. Since UTEW in \figurename\ref{algorithm:universal} works for all competitors, we derive separate bounds for $p^*$ and $q_t^*$ from Theorem \ref{theorem:universal_regret} and its proof \\(for $\theta''=2+\theta'$, $C' = 4c+2c'\sqrt{\theta''}$, $C''=4\sqrt{2}c+2c'\sqrt{3\theta''-2-\phi\theta''}$).
\end{proof}

\newpage
\section{Discounted Losses}
\label{section:discounted_losses}
Notice that our approach can be adapted for discount loss scenarios (e.g. with geometric progression). 
Consider the definition of ``discounted regret''

\begin{equation} \label{eq:discounted_regret}
	R_T^{(dis)} = \sum_{t=1}^{T} \beta_{T-t} l_t^\top(p_t - p_t^*),
\end{equation}
where $\beta_t = \beta_0 \alpha^t$ for some $\beta_0,\alpha>0$. Possibly, it is also the case that $\alpha<1$ due to the nature of `discounting'.

Since our algorithms work in a scale-free manner, we can just use the loss vectors $\alpha^{1-t} l_t$ and bound the regret in \eqref{eq:discounted_regret}. This upper-bound would still be optimal (min-max sense) in accordance with our regret guarantee theorems/corollaries and Corollary \ref{corollary:dynamic_regret_lower_bound} for the adversarial lower-bound.

\section{Conclusion}
\label{section:conclusion}
In this paper, we have investigated the problem of prediction with expert advice under a non-stationary setting. This dynamic nature results in time-varying loss ranges and best combination of experts. We have achieved adaptive tracking of the best expert combination, even under certain restrictions.
We have shown that our results (regret bounds) are min-max optimal and cannot be improved further in the general sense.
Our algorithms also perform in a universal manner such that they achieve the specific min-max optimal guarantee jointly specified by the nature of both the causally observed losses and the obviously unknown competitor sequence.

\bibliographystyle{IEEEtran}

\newpage
\bibliography{IEEEabrv,../bib/paper}

\newpage
\appendices

\section{Proofs of Section \ref{section:adversarial_lower_bounds}} \label{appendix:adversarial_lower_bounds}
\subsection{Proof of Lemma \ref{lemma:two_expert_regret_lower_bound}}
\begin{proof}
	Given the loss range sequence $\{|l_{t,1} - l_{t,2}|\}_{t=1}^T$, set $U_t = |l_{t,1} - l_{t,2}|/2$. Assume $l_{t,1},l_{t,2} \in \{-U_t,U_t\}$ without loss of generality. For the min-max analysis, we shall replace the $\max$ operations with lower bounds similar to \cite{abernethy2008optimal}. Thus, for adversarial loss sequences $\{l_{t,1}\}_{t=1}^T$ and $\{l_{t,2}\}_{t=1}^T$ with worst-case regret $\overline{R_T}$,
	\begin{equation*}
		\overline{R_T}
		\geq \E{\sum_{t=1}^T \sum_{m=1}^2 l_{t,m}(p_{t,m} - p_{*,m})},
	\end{equation*}
	where $l_{t,m}$ are set in a random manner such that $l_{t,1} = -l_{t,2}$ and $l_{t,1} = B_t U_t$ with $B_t \in \{-1,1\}$. Each random variable $B_t$ is uniformly sampled from $\{-1,1\}$ in an independently and identically distributed manner. Thus, $\E{l_{t,m} p_{t,m}} = 0$ for all $t,m$. Then, since $p_*$ is the best expert distribution with minimum loss, it is a one-hot vector. Consequently,
	\begin{align*}
		\overline{R_T}
		&\geq -\E{\min_{p_*} \sum_{t=1}^T \sum_{m=1}^2 l_{t,m}p_{*,m}} 
		\geq \E{\max\left\{\sum_{t=1}^T B_tU_t, -\sum_{t=1}^T B_tU_t\right\}}
		\geq \E{\left|\sum_{t=1}^T B_tU_t\right|}
	\end{align*}
	and, after using Khintchine’s inequality \cite{hitczenko1994rademacher},
	\begin{equation*}
		\overline{R_T} \geq \frac{1}{\sqrt{2}} \sqrt{\sum_{t=1}^{T} U_t^2} = \frac{1}{\sqrt{2}} L_T^{(\infty,2)}.
	\end{equation*}
\end{proof}

\newpage
\subsection{Proof of Theorem \ref{theorem:static_regret_lower_bound}}
\begin{proof}
	Due to the nature of feasible decisions (i.e. they belong to a simplex), the loss vectors $l_t$ can be freely translated (shifted with a scalar) without any effect on the regret. Thus, assume that, for $$U_{t} = \min_{\mu \in \mathbb{R}} \max_{1\leq k\leq M} |l_{t,k}-\mu|,$$ in an adversarial fashion, the loss entries can be in the range $[-U_t, U_t]$ without loss of generality. For the largest $d$ such that $M \geq 2^d$, separate $M$ experts into two sets $A = \{1,\ldots,2^d\}$ and $A'=\{m: m\in \mathbb{N}, 2^d< m\leq M\}$. Set the losses of all $m$ in $A'$ to $U_t$. For each $m$ in $A$, generate the binary indicator vector $I_m = [I_{m,0}, \ldots, I_{m,d-1}]$ such that each element $I_{m,n}$ is in $\{0,1\}$ and $\sum_{n=0}^{d-1} I_{m,n} 2^n = m-1$. Hence, each vector $I_m$ uniquely identifies its corresponding expert $m$. Now, separate the overall time horizon into $\min\{d,T\}$ different games in an adversarial manner. In these games, the loss incurred by the $m^{th}$ expert at some time $t$ is determined by a coin flip $B_t \in \{-1,1\}$ and the $n_t^{th}$ element of the binary indicator vector $I_m$ (corresponding to the said expert), where $n_t$ identifies the ongoing game at time $t$. Specifically, the losses are chosen as $l_{t,m} = B_t U_t (2 I_{m,n_t} - 1)$. Effectively, in each of these games, there is $2$ distinct experts, and identifying the best expert among the two in all of these games corresponds to identifying the best expert among the original $M$ different experts. According to Lemma \ref{lemma:two_expert_regret_lower_bound}, the regret of each game is lower bounded by $L_T^{(\infty,2)}/\sqrt{2\min\{d,T\}}$ for the adversarial sequence of $\{U_t\}_{t=1}^\infty$ for which $\left(L_T^{(\infty,2)}\right)^2$ splits into $\min\{d,T\}$ different equal sums. Summing together each game, and lower-bounding $d$ with $\lfloor \log_2 M \rfloor$, the overall game of selecting among $M$ experts has its regret lower-bounded by $L_T^{(\infty,2)}\sqrt{\min\{\lfloor \log_2 M \rfloor,T\}/2}$. 
\end{proof}

\subsection{Proof of Corollary \ref{corollary:dynamic_regret_lower_bound}}
\begin{proof} 
	First, we generate the lower-bound for fixed competitor $p_t^* = p_*$, where $p_*$ is a one-hot vector from Theorem \ref{theorem:static_regret_lower_bound}. Then, we generate $\lfloor P_T + 1\rfloor$ separate games (if $\lfloor P_T + 1\rfloor\lfloor \log_2 M\rfloor \leq T$) by splitting the time horizon in an adversarial manner, where $\lfloor \cdot \rfloor$ is the floor operation, and sum the individual lower bounds for each game, similar to the proof of Theorem \ref{theorem:static_regret_lower_bound}. If $\lfloor P_T + 1\rfloor\lfloor \log_2 M\rfloor > T$, split into a total of $T$ two-expert games instead. Either way, we obtain the claim since $\sum_{t=1}^{T-1} 0.5\|p_{t+1}^* - p_t^*\|_1 \leq (\lfloor P_T + 1\rfloor - 1) \leq P_T$.
\end{proof}

\newpage
\section{Proofs of Section \ref{section:exponential_weighting_with_uniform_mixing}}
\label{appendix:exponential_weighting_with_uniform_mixing}
\subsection{Proof of Lemma \ref{lemma:KL}}
\begin{proof}
	Considering \eqref{eq:upd_uni}, we begin by noting that
	\begin{align} \label{eq:KL}
		D(p_t^* \| p_t) &- D(p_t^* \| p_{t+1}) = \sum_{m=1}^{M} p_{t,m}^* \log\left(\frac{p_{t+1,m}}{p_{t,m}}\right) 
		\geq \Ep{-\eta_t(l_{t,m} - \mu_t) - \log Z_t + \log(1-\gamma_t)}{p_{t}^*},
	\end{align}
	since $\log$ is monotonically increasing. Here $Z_t$ is a short-hand with $Z_t = \sum_{k=1}^{M} p_{t,k} e^{-\eta_t (l_{t,k} - \mu_t)}$. Note that here, we have introduced a scalar $\mu_t$. This is possible because replacing $l_{t,m}$ with $(l_{t,m}-\mu_t)$ is perfectly plausible as it only means multiplying both the numerator and denominator of the first component in the right-hand side of \eqref{eq:upd_uni} with $e^{\eta_t \mu_t}$.
	
	We then lower bound $-\log Z_t$ by upper bounding $\log Z_t$, under our assumption $|\eta_t(l_{t,m} - \mu_t)| \leq 1$, as
	\begin{align} \label{eq:Z_t}
		\log(Z_t) &= \log\left(\Ep{e^{-\eta_t (l_{t,m} - \mu_t)}}{p_t}\right)
		\leq -\eta_t\Ep{l_{t,m} - \mu_t}{p_t} + \eta_t^2 \Ep{(l_{t,m} - \mu_t)^2}{p_t}
		\leq \eta_t^2 \Ep{(l_{t,m} - \mu_t)^2}{p_t}
	\end{align}
	using the facts that $\log x \leq x-1$; for $|x|\leq  1$, we have $e^x \leq 1 + x + x^2$; and $\mu_t = \Ep{l_{t,m}}{p_t}$.
	
	Further lower bounding \eqref{eq:KL} using \eqref{eq:Z_t}, which is followed with division of both sides  by $\eta_t$, results in the initial claim after substituting $\mu_t$ with $l_t^\top p_t$.
\end{proof}

\subsection{Proof of Lemma \ref{lemma:change_KL}}
\begin{proof}
	Trivially,
	\begin{align*}
		-D(p_t^* \| p_{t+1}) = -D(p_t^* \| p_{t+1}) 
		&+D(p_{t+1}^* \| p_{t+1}) 
		-D(p_{t+1}^* \| p_{t+1}).
	\end{align*}
	Also, using Definitions \ref{definition:entropy} and \ref{definition:KL_divergence},
	\begin{align*}
		&D(p_{t+1}^* \| p_{t+1}) -D(p_t^* \| p_{t+1}) 
		= H(p_t^*) - H(p_{t+1}^*) 
		+\sum_{m=1}^{M} (p_{t+1,m}^* - p_{t,m}^*) \log\left(\frac{1}{p_{t+1,m}}\right).
	\end{align*}
	The component $\sum_{m=1}^{M} (p_{t+1,m}^* - p_{t,m}^*) \log(1/p_{t+1,m})$ can be upper-bounded by considering the fact that since $p_t^*$ can arbitrarily change in successions, all increases can correspond to $m$ with the lowest $p_{t+1,m}>0$ and all decreases can correspond to $m$ with the highest $p_{t+1,m}\leq 1$. Hence, since $\log 1 = 0$ and non-negative $\log(1/p_{t+1,m})$ grows arbitrarily large as $p_{t+1,m} > 0$ gets smaller, we can tightly upper bound the aforementioned sum by ignoring the sum components with $p_{t+1,m}^*-p_{t,m}^*\leq 0$ (i.e. replace with $0$) and upper-bound the non-negative multipliers $\log(1/p_{t+1,m})$ of the positive quantities $(p_{t+1,m}^* - p_{t,m}^*)$ by their maximum, (i.e. replace positive $p_{t+1,m}$ in $\log(1/\cdot)$ with $\min_m p_{t+1,m}$). After combining everything, we obtain the claimed result since $$\sum_{m=1}^{M} \max\{0,p_{t+1,m}^* - p_{t,m}^*\} = 0.5\|p_{t+1}-p_t\|_1.$$
\end{proof}

\subsection{Proof of Corollary \ref{corollary:non_increasing}}
\begin{proof}
	Apply telescoping sum for $(H(p_t^*) - H(p_{t+1}^*))/\eta_t$ and $(D(p_t^* \| p_t) - D(p_{t+1}^* \| p_{t+1})/\eta_t$, where $$\max H(p_t^*) \leq \log M$$ and $\max D(p_t^* \| p_t) \leq \log(M/\gamma_T)$, as $p_{t,m} \geq \gamma_T/M$, $\forall t$. For $\|p_{t+1}^* - p_t^*\|_1$, consider $p_{T+1} = p_T$ since the regret is independent of $p_{T+1}$.
\end{proof}

\subsection{Proof of Corollary \ref{corollary:setting_gamma_t}}
\begin{proof}
	The terms with $\log(1/\gamma_t)$ are computed via straight substitution. We need to compute $$\sum_{t=1}^{T}-\log(1-\gamma_t)/\eta_t.$$ We can upper bound this by replacing $\eta_t$ with the lower bound $\eta_T$ since $-\log(1-\gamma_t) \geq 0$ for $\gamma_t \geq 0$. Then, the computation transform into $(1/\eta_T) \log\left(\prod_{t=1}^{T} (1-\gamma_t)\right)$. Afterwards, setting $\gamma_t = (t+1)^{-1}$ results in $$\log(T+1)/\eta_T.$$
\end{proof}

\subsection{Proof of Theorem \ref{theorem:learning_rate_optimization}}
\begin{proof}
	The second term on the right hand side of the regret inequality comes from the $\min$ operator in $\eta_t$ for which $E_t$ is non-decreasing. This application is necessary to ensure $\eta_t(l_{t,m}-\mu_t) \leq 1$ for two-sided losses. The first term is realized after computing the telescoping sum by upper bounding as $\eta_t \leq \widetilde{\eta}_t$ since variances are non-negative. 
\end{proof}

\subsection{Proof of Corollary \ref{corollary:minimum_biasing}}
\begin{proof}
	The proof directly follows the derivations so far. We only modify the assumption of Lemma \ref{lemma:KL} such that $|\eta_t(l_{t,m}-\mu_t)|\leq 1$ with $\eta_t(l_{t,m}-z_t)\geq 0$ by using $z_t$ instead of $\mu_t$. Since $\eta_t\geq 0$, the new assumption can be readily achieved without considering $(1/E_t)$ as an upper-bound for $\eta_t$. 
	
	Furthermore, again in Lemma \ref{lemma:KL}, $\log(Z_t)$ can be upper-bounded by $-\eta_t \Ep{l_{t,m}-z_t}{p_t} + 0.5\eta_t^2\Ep{(l_{t,m}-z_t)^2}{p_t}$ instead, using the fact that $e^x \leq 1 + x + x^2/2$ for non-positive $x$. When combined with other terms, after some cancellations, this results in the replacement of $\eta_t \Ep{(l_{t,m}-\mu_t)^2}{p_t}$ with $0.5\eta_t \Ep{(l_{t,m}-z_t)^2}{p_t}$ in all the subsequent derivations, which produces the regret
	\begin{align*}
		R_T \leq &\left(\sum_{t=1}^{T} \frac{\eta_t}{2} \Ep{(l_{t,m} - z_t)^2}{p_t}\right) +\frac{2\log(M(T+1))}{\eta_T}
		+\frac{\log\left(MT\right) \sum_{t=1}^{T-1}\|p_{t+1}^* - p_t^*\|_1}{2\eta_T},
	\end{align*}
	for which the suitable learning rate is as claimed, along with the result.
\end{proof}

\section{Proofs of Section \ref{section:exponential_weighting_with_truncation}}
\label{appendix:exponential_weighting_with_truncation}
\subsection{Proof of Lemma \ref{lemma:truncation}}
\begin{proof}
	Due to the relation between $p_{t+1}$ and $q_{t+1}$ as detailed in \eqref{eq:truncate}, we have
	\begin{align*}
		D(p_t^* \| q_{t+1}) - D(p_t^* \| p_{t+1}) 
		&= \sum_{m=1}^M p_{t,m}^* \log\left(\frac{p_{t+1,m}}{q_{t+1,m}}\right)
		\geq \sum_{m=1}^M p_{t+1,m} \log\left(\frac{p_{t+1,m}}{q_{t+1,m}}\right)
		\geq D(p_{t+1}\|q_{t+1})\geq 0,
	\end{align*}
	using the rearrangement inequality, since
	\begin{align*}
		\log(p_{t+1,m}/q_{t+1,m})\geq\alpha 
		&\text{ if } 
		p_{t+1,m}=a,\\
		\log(p_{t+1,m}/q_{t+1,m})\leq\alpha 
		&\text{ if } 
		p_{t+1,m}=b,\\
		\text{ and }
		\log(p_{t+1,m}/q_{t+1,m})=\alpha &\text{ otherwise.}
	\end{align*}
\end{proof}

\subsection{Proof of Theorem \ref{theorem:vanilla_truncation}}
\begin{proof}
	From Lemma \ref{lemma:KL}, Theorem \ref{theorem:initial_regret_bound} and Lemma \ref{lemma:truncation},
	\begin{align*}
		R_T &\leq \sum_{t=1}^{T} l_t^\top (p_t - p_t^*) 
		\begin{multlined}
			\leq \sum_{t=1}^{T} \eta_t \Ep{(l_{t,m} - \mu_t)^2}{p_t}
			+\frac{D(p_t^* \| p_t) - D(p_t^* \| p_{t+1})}{\eta_t}
		\end{multlined},
	\end{align*}
	since $\gamma_t=0$, we have $q_{t+1}$ corresponding to the same quantity as $p_{t+1}$ in Lemma \ref{lemma:KL} and Theorem \ref{theorem:initial_regret_bound} when $\gamma_t=0$, and $q_{t+1}$ is then replaced by $p_{t+1}$ in here using Lemma \ref{lemma:truncation}.
	
	Then, since $\eta_t>0$, by using Lemma \ref{lemma:change_KL}, we get
	\begin{align*}
		R_T \leq &\sum_{t=1}^{T} \eta_t \Ep{(l_{t,m} - \mu_t)^2}{p_t} +\frac{\|p_{t+1}^* - p_t^*\|_1 \log(1/a)}{2\eta_t} 
		+\frac{H(p_t^*) - H(p_{t+1}^*) + D(p_t^* \| p_t) - D(p_{t+1}^* \| p_{t+1})}{\eta_t},
	\end{align*}
	similar to Corollary \ref{corollary:parametric_regret_bound}. Following that, as $H(p_t^*)\leq \log(M)$ and $D(p_t^* \| p_t)\leq \log(1/a)$ from Definitions \ref{definition:entropy} and \ref{definition:KL_divergence}, we can simplify by telescoping the sum for non-increasing $\eta_t>0$. Similar to Corollary \ref{corollary:non_increasing}, this gives
	\begin{align*}
		R_T \leq \sum_{t=1}^{T} &\eta_t \Ep{(l_{t,m} - \mu_t)^2}{p_t} 
		+\frac{\log M + \log(1/a)}{\eta_T}
		+\frac{\log(1/a) \sum_{t=1}^{T-1}\|p_{t+1}^* - p_t^*\|_1}{2\eta_T}
		.
	\end{align*}
	Similarly with Theorem \ref{theorem:learning_rate_optimization}, setting the given learning rates results in the claim.
\end{proof}

\newpage
\subsection{Proof of Corollary \ref{corollary:eliminate_logT}}
\begin{proof}
	We have
	\begin{align*}
		R_T 
		&\leq \sum_{t=1}^{T} l_t^\top(p_t - p_t^*) 
		= \sum_{t=1}^{T} \frac{1}{1-\alpha}l_t^\top(d_t - d_t^*),
	\end{align*}
	where $d_{t,m},d_{t,m}^* \in [\alpha/M, \alpha/M + (1-\alpha)]$ are lower-bounded with the competitor cumulative total variation distance as $D_T = 0.5\sum_{t=1}^{T-1} \|d_{t+1}^* - d_t^*\|_1 = (1-\alpha)P_T$.
	Plugging $d_t$ as $p_t$, $D_T$ as $P_T$ and $l_t/(1-\alpha)$ for Theorem \ref{theorem:vanilla_truncation}, we get
	\begin{equation*}
		R_T \leq 2\sqrt{2(\log(M^2/\alpha)+\log(M/\alpha)D_T)V_T^d/(1-\alpha)^2} 
		+(\log(M^2/\alpha)+\log(M/\alpha)D_T)E_T^d/(1-\alpha),
	\end{equation*}
	where $V_t^d = \sum_{\tau=1}^t \Ep{(l_{\tau,m} - \mu_\tau^d)^2}{d_\tau}$, $\mu_t^d = \Ep{l_{t,m}}{d_t}$, and $E_t^d =  \max_{1\leq \tau\leq t,1\leq m\leq M} |l_{\tau,m} - \mu_\tau^d| \leq E$ as in Theorem \ref{theorem:vanilla_truncation}.
	
	Note that the terms $V_t^d,E_t^d$ are not much related to the underlying problem setting. Thus, we instead utilize Corollary \ref{corollary:alternative_truncation}, which, after substituting for $D_T$, gives
	\begin{equation*}
		R_T \leq 2\sqrt{\left(\frac{\log(M^2/\alpha)}{(1-\alpha)^2}+\frac{\log(M/\alpha)}{1-\alpha}P_T\right)Q_T^d},
	\end{equation*}
	where, for $z_t = \underset{1\leq m\leq M}{\min} l_{t,m}$, $Q_t^d = \sum_{\tau=1}^t \Ep{(l_{\tau,m} - z_\tau)^2}{d_\tau}$ is unknown at the start. 
	
	Ignoring $Q_T^d$, since it is too convoluted and unknown, we can numerically extract 
	$$\widehat{\alpha} = \arg\min_{0<\alpha<1} \left(\frac{\log(M^2/\alpha)}{(1-\alpha)^2}+\frac{\log(M/\alpha)}{1-\alpha}P_T\right).$$
	
	For the analysis, we can use a substitute $\alpha^*$, which can also be used to produce the same guarantees we will generate. First, we group with the greater one, $\log(M^2/\alpha)/(1-\alpha)^2$. After some derivative optimizations on $\log(M^2/\alpha)/(1-\alpha)^2$, using the fact that $M\geq 2$, we get
	\begin{equation*}
		R_T \leq \sqrt{\frac{2}{\alpha^*(1-\alpha^*)}\left(1+(1-\theta) P_T\right)Q_T^d},
	\end{equation*}
	for $0\leq \theta = \alpha^*(1 + 2\log(M))$ and $0<\alpha^*<1$ such that $\alpha^*(4\log(M) -2\log(\alpha^*)+1)=1$, which implies
	$$\alpha^* = \frac{-0.5}{W_{-1}\left(-0.5e^{-0.5}M^{-2}\right)}.$$
	
	Here, $W_{-1}(\cdot)$ is the analytical continuation of the Lambert W function for negative arguments.
	Correspondingly, $2/(\alpha^*(1-\alpha^*)) = -4W + 4W/(2W+1)$, where $W$ is a shorthand with $W = W_{-1}\left(-0.5e^{-0.5}M^{-2}\right)$. 
	
	\newpage
	Note that as $M>0$ increases, so does $\left(-0.5e^{-0.5}M^{-2}\right)$, which decreases $W$, since, when $M\geq 2$ and consequently $-0.5e^{-0.5}M^{-2}\geq -e^{-1}$, $W_{-1}(\cdot)$ is monotonically decreasing. Then, as $W\leq -0.5$, $4W/(2W+1)$ is maximized when $M=2$, i.e., at its minimum. Since $\log$ is increasing, we have $4W/(2W+1)\leq r\log(M)$ for all $M$, for some $r$, with equality at $M=2$, where empirically $r\approx 3.30305083218$. 
	
	Next, consider $-W\leq k\log M$ for all $M\geq 2$ and some positive $k$. Alternatively, $-k\log(M)M^{2-k} \geq -0.5e^{-0.5}$ due to the nature of $W$ and $M>0$. This is not satisfied for large $M$ when $k\leq2$. For $k> 2$, left side has a single minimum at $\log(M)$ being $\max\{\log2,1/(k-2)\}$ since $M\geq 2$. For the case when $1/(k-2)>\log(2)$, $k<2+1/\log(2)$. However, we already assumed that $k\geq -W/\log(M)$ for all $M$, i.e., by considering $M=2$ empirically, $k\geq 5 > 4 \geq 2+1/\log(2)$, which is a contradiction. If we investigate $1/(k-2)\leq\log2$, which achieves the minimum at $\log(M)=\log(2)$, we see that, empirically, $k\approx 5.70474368345$. 
	
	In the end, we can simplify as the claimed using $$c=\sqrt{4k+r}\approx\sqrt{26.122025566}\approx 5.11097109814.$$
	
	By following similar arguments for $\theta$, we can see that $$\theta=\alpha^*(1+2\log(M))\geq -(0.5+\log(2))/W_{-1}(-0.5e^{-0.5}/4).$$ Empirically, $\theta\approx 0.3017396777$.
\end{proof}

\newpage
\section{Proofs of Section \ref{section:universal_competition}}
\label{appendix:universal_competition}
\subsection{Proof of Corollary \ref{corollary:doubling_rates}}
\begin{proof}
Notice $P_k = 2^{k-1}-1$ is a trivial bound for the cumulative total variation distance from $t=2^{k-1}$ to $t=2^{k}-1$, since $0.5\sum_{t=2^{k-1}}^{2^k-2} \|p_{t+1}^*-p_t^*\|_1 \leq 2^{k-1}-1$ due to the fact that $\|p_{t+1}^*-p_t^*\|_1\leq 2$. Denote as $R_{(k)}$ the regret incurred from $t=2^{k-1}$ to $t=2^k-1$ for $k\leq K$ and from $t=2^K$ to $t=T$ for $k=K+1$. Then, 
\begin{align*}
	R_T &= \sum_{k=1}^{K+1} R_{(k)} \leq \sum_{k=1}^{K+1} c\sqrt{(1+(1-\theta)P_k)Q_{(k)}^d\log(M)} 
	\leq c\sqrt{\log(M)\left(K+1 + (1-\theta)\sum_{k=1}^{K+1} P_k\right)\left(\sum_{k=1}^{K+1} Q_{(k)}^d\right)}
	\\&\leq c\sqrt{\left(3-\theta +\left(4-4\theta+\frac{\theta}{\log2}\right)P_T\right)Q_T^d\log(M)}.
\end{align*}
The first bound comes from individual regret bounds from Corollary \ref{corollary:eliminate_logT}. The second comes from Cauchy-Schwartz inequality. The third holds since $2^{K-1}\leq 1+P_T$, and $\sum_{k=1}^{K+1} P_k = 2^{K+1} -1 -K-1\leq 3+4P_T-K-1$, and $\sum_{k=1}^{K+1} Q_{(k)}^d=Q_T$ by definition. After further adjustments, we achieve the claim.
\end{proof}

\subsection{Proof of Theorem \ref{theorem:universal_regret}}
\begin{proof}
	Using the construct in \figurename\ref{algorithm:universal}, we can say the following
	\begin{align*}
		R_T 
		&\leq \sum_{t=1}^T l_t^\top \left(p_t - p_t^{(A_K)}\right) + \sum_{t=1}^T l_t^\top \left(p_t^{(A_K)} - p_t^*\right).
	\end{align*}
	Since $l_{t,(B_0)} = l_t^\top p_t$ and $l_{t,(A_k)} = l_t^\top p_t^{A_K}$,
	\begin{align*}
		R_T \leq &\sum_{t=1}^T \left(\sum_{k=0}^{K-2} \left(l_{t,(B_{k})} - l_{t,(B_{k+1})}\right)\right) +\sum_{t=1}^T\left(l_{t,(B_{K-1})} - l_{t,(A_{K})}\right) 
		+ l_t^\top \left(p_t^{(A_K)} - p_t^*\right)
	\end{align*}
	where $l_{t,(B_k)} = p_{t}^{(B_{k})} l_{t,(A_{k+1})} + p_{t,2}^{(B_{k})} l_{t,(B_{k+1})}$. Hence, out of all, there exist $K+1$ number of concurrently running prediction with expert advice setups contributing to the overall regret. $K$ of these have $M=2,P_T=0$ according to Corollary \ref{corollary:eliminate_logT} and a single one has $M=M^{true}$ and $P_T=P_K$ with $P_T^{true}\leq P_K< 1+2P_T^{true}$ in accordance with Corollary \ref{corollary:doubling_rates}. The regret then becomes (for $M\geq 2$)
	\begin{align*}
		R_T &\leq 2c\left(2+\log(1+P_T)\right)\sqrt{\log2}L_T^{(\infty,2)}
		+2c'\sqrt{\left(1+(1+\theta')(1+2P_T)\right)\log(M)}L_T^{(\infty,2)}
		\\&\leq C\sqrt{\lfloor1+P_T\rfloor\log(M)}L_T^{(\infty,2)},
	\end{align*}
	where $C= 4\sqrt{2}c + 2c'\sqrt{4+3\theta'}\approx 75.1691252097$. Utilized facts for the derivations are as follows: 
	
	$\lfloor1+P_T\rfloor \geq q+(1-q)P_T$ for all $1\geq q\geq0$ and $P_T\geq 0$, 
	
	$2+\log(1+P_T)\leq 2\sqrt{2\lfloor 1+P_T \rfloor}$, 
	
	$L_T^{(\infty,2)} = \sqrt{\sum_{t=1}^T \min_{\mu \in \mathbb{R}} \max_{1\leq k\leq M} |l_{t,k}-\mu|^2}$,
	
	$Q_t^d\leq 2L_T^{(\infty,2)}$ for all $A_k$ and $B_k$.
\end{proof}

\ifCLASSOPTIONcaptionsoff
  \newpage
\fi

\end{document}